\documentclass[sigconf]{acmart}

\usepackage{amssymb}
\usepackage{amsmath}
\usepackage{algorithm}
\usepackage{algorithmic}
\usepackage{subfigure}
\usepackage{makecell}
\usepackage{multirow}
\usepackage{adjustbox}

\newtheorem{lemma}{Lemma}
\usepackage[title]{appendix}

\newcommand{\R}{\mathbb{R}}

\AtBeginDocument{%
  \providecommand\BibTeX{{%
    \normalfont B\kern-0.5em{\scshape i\kern-0.25em b}\kern-0.8em\TeX}}}

\copyrightyear{2025}
\acmYear{2025}
\setcopyright{acmlicensed}
\acmConference[KDD '25] {Proceedings of the 31st ACM SIGKDD Conference on Knowledge Discovery and Data Mining V.1}{August 3--7, 2025}{Toronto, ON, Canada.}
\acmBooktitle{Proceedings of the 31st ACM SIGKDD Conference on Knowledge Discovery and Data Mining V.1 (KDD '25), August 3--7, 2025, Toronto, ON, Canada}
\acmISBN{979-8-4007-1245-6/25/08}
\acmDOI{10.1145/3690624.3709160}
\begin{document}

\title{Fairness without Demographics through \\ Learning Graph of Gradients}

\author{Yingtao Luo}
\affiliation{
  \institution{Carnegie Mellon University}
  \city{Pittsburgh}
  \country{USA}
  }
\email{yingtaol@andrew.cmu.edu}

\author{Zhixun Li}
\affiliation{
  \institution{The Chinese University of Hong Kong}
  \city{Hong Kong}
  \country{China}
  }
\email{zxli@se.cuhk.edu.hk}

\author{Qiang Liu}
\authornote{To whom correspondence should be addressed.}
\affiliation{
  \institution{CRIPAC, MAIS,\\Institute of Automation, Chinese Academy of Sciences}
  \city{Beijing}
  \country{China}
  }
\email{qiang.liu@nlpr.ia.ac.cn}

\author{Jun Zhu}
\affiliation{
  \institution{Department of Computer Science and Technology, Institute for AI, THBI Lab, BNRist Center, Tsinghua-Bosch Joint ML Center, Tsinghua University}
  \city{Beijing}
  \country{China}
  }
\email{dcszj@tsinghua.edu.cn}



\begin{abstract}
Machine learning systems are notoriously prone to biased predictions about certain demographic groups, leading to algorithmic fairness issues. Due to privacy concerns and data quality problems, some demographic information may not be available in the training data and the complex interaction of different demographics can lead to a lot of unknown minority subpopulations, which all limit the applicability of group fairness. Many existing works on fairness without demographics assume the correlation between groups and features. However, we argue that the model gradients are also valuable for fairness without demographics. In this paper, we show that the correlation between gradients and groups can help identify and improve group fairness. With an adversarial weighting architecture, we construct a graph where samples with similar gradients are connected and learn the weights of different samples from it. Unlike the surrogate grouping methods that cluster groups from features and labels as proxy sensitive attribute, our method leverages the graph structure as a soft grouping mechanism, which is much more robust to noises. The results show that our method is robust to noise and can improve fairness significantly without decreasing the overall accuracy too much.
\end{abstract}

\maketitle

\section{Introduction}
Fairness in machine learning has become an urgent concern, as machine learning systems can be biased against certain demographic groups, which contributes to the socioeconomic disparities in many areas such as healthcare \cite{gianfrancesco2018potential}, finance \cite{hajian2016algorithmic}, etc. For example, when learning the default risk of financial loan applicants, due to certain biases, the model prediction can be inaccurate for some protected groups, such as minorities. To deal with this issue, some approaches \cite{hashimoto2018fairness, sagawa2019distributionally, lahoti2020fairness} follow the Rawlsian Max-Min fairness \cite{rawls2004theory} to minimize the worst-case errors over all groups. Some other methods minimize the prediction accuracy gap between different demographic groups \cite{creager2021environment, rahman2022fair, chaifairness}. Most of these existing methods require sensitive attributes, such as race, gender, etc., to identify which group is discriminated against by machine learning models. However, due to privacy concerns, these sensitive attributes are often not accessible to algorithm developers. For example, certain regulations like the HIPAA privacy rule have set up safeguards, which now limit not only direct access but also indirect methods that infer sensitive attributes, as these can lead to re-identification risks and other privacy concerns. An additional challenge arises when dealing with numerous demographic attributes. Defining the intricate interplay among various demographic factors for minority subgroups beforehand can be complex, and the potential protected groups increase exponentially as the number of sensitive attributes grows. This, in turn, escalates the difficulty of identifying the most disadvantaged group. Consequently, it is crucial to advocate for machine learning fairness that does not rely on demographic information.

To ensure fairness without demographics, many existing methods with \textbf{proxy sensitive attributes} \cite{yan2020fair, grari2021fairness, du2021fairness, zhao2022towards} assume the \textit{correlation between sensitive attributes (groups) and non-sensitive attributes (features)}, perform clustering to obtain surrogate groups, and enforce group fairness. The problem with these methods, however, is the \textbf{difficulty to guarantee a large overlap of the proxy attributes with the real protected groups}, especially when many protected groups are unknown and the distributional discrepancies between sensitive attributes are large \cite{chaifairness}. Other methods such as ARL \cite{lahoti2020fairness} generate weights for different samples, but they can be \textbf{susceptible to noises (e.g., mislabelling)} when outliers are given superior weights due to their rarity in the data, and thus may lead to severe degradation of fairness metrics.

When these existing works focus on the correlation between features and groups, we argue that the model training information is also valuable for fairness. When a model encounters an underrepresented sample, we regard the model as uncertain if a substantial gradient updates the model parameters due to limited exposure to similar samples during training. In other words, the gradients of a model combine information of both \textit{input features and model error between different demographic groups}. We show that \textbf{gradients are more effective in representing sensitive attributes}, assuming that input features and model accuracy are strongly correlated. The assumption holds true widely, as model accuracy is largely determined by the quality of feature. As neural networks consist of many neural components, we also demonstrate that the last-layer gradient is sufficient for representing demographics.

In this paper, we develop an adversarial learning framework, which has a learner for the main task and an adversary for generating the sample weights to maximize the loss of the learner network. Then, we assign these weights to the minimization of learner loss. 
To alleviate the issues with surrogate groups and outlier noises, we leverage the gradients of the learner for grouping samples by demographics and construct a \textbf{graph of gradients (GoG)} where each sample is connected to the K-nearest samples with similar gradients. 
Through this graph, the weight of each sample is calculated as the aggregation of neighboring samples. This \textit{soft-grouping mechanism} mimics the effect of grouping similar samples without enforcing hard boundaries between demographic groups. 
Since the weight of each sample will be spread to neighbors, we \textbf{avoid assigning superior weights to noisy outliers}. In experimental evaluations, we verify that our method can successfully improve both fairness and the trade-off between fairness and accuracy significantly. 

To summarize, our contributions are listed as follows:
\begin{itemize}
\item We propose a fairness-without-demographics algorithm to mitigate the machine learning unfairness issue. The algorithm can scale up to large datasets of diverse complexities, structures and domains well.
\item We show that the gradients of a machine learning model are more effective in representing demographic groups than input features, which holds true as long as the deep learning learner is better than random guessing. We also show that the last-layer gradients are sufficient.
\item We propose to identify demographic subgroups by a graph of gradients (GoG) as a soft-grouping method, where each sample is connected to the K-nearest samples with similar gradients. GoG can address the issues with noisy outliers and is extremely efficient to train with linear complexity.
\item Extensive experiments on three public datasets and five baselines show that our method outperforms other algorithms significantly in both fairness and accuracy. The robustness study suggests that our method is more robust against noises. 
\end{itemize}

\section{Related Work}

\subsection{Group fairness for classification}
Group fairness is a concept that aims to ensure that the outcomes of an algorithm are equitable across different subpopulations defined by sensitive attributes, such as race, gender, etc. To alleviate the group disparity \cite{jiang2022generalized}, Equal Opportunity \cite{hardt2016equality} hopes that the true positive rates should be the same for all subpopulations, and Predictive Equality \cite{chouldechova2017fair} requires the equality of false positive rates. Preprocessing methods \cite{chen2018my, jang2021constructing} ensure that the data used for training is unbiased and representative of different subgroups by re-sampling, feature selection, etc. In-processing methods \cite{madras2018learning, iosifidis2019adafair, roh2021fairbatch, chaifairness, chai2022fairness} regularize the training process with fair constraints, sample reweighting, and adversarial training. Post-processing methods \cite{pleiss2017fairness, kimlearning, jang2022group} focus on adjusting the model prediction after training by threshold adjustment, calibration, etc., which are usually very efficient. However, to guarantee group fairness, the availability of sensitive information is a necessity. Some papers \cite{celis2021fair, celis2021fairb, giguere2022fairness} also address fairness concerns by using techniques that are robust to noisy or shifting sensitive attributes. Others propose methods to address fairness concerns in complex data structures \cite{dong2022edits, wang2022improving, dong2023fairness, dong2023interpreting, li2024rethinking}. While certain policies restrict versatility, we still need to count on fairness without demographic methods. 

\subsection{Fairness without demographics}
In light of the challenges to discovering the worst-off groups due to both the regulatory limitations and the complex interaction of many demographic variables \cite{shui2022learning}, increasing methods are proposed in recent years to achieve fairness without demographics. Some methods follow the Rawlsian Max-Min fairness \cite{rawls2004theory} to minimize the empirical risk of the group with the least utility. For example, Distributionally Robust Optimization (DRO) \cite{hashimoto2018fairness} proposes to use $\chi^2$-divergence to discover and minimize the worst-case distribution repeatedly, which essentially only focuses on the learning of the worst-off group. Adversarial Reweighted Learning (ARL) \cite{lahoti2020fairness} uses an adversary network to generate sample weights that maximize the empirical risk and performs weighted learning for the learner model. Based on the concept of computational identifiability, ARL hypothesizes that it can learn demographic information from data features and labels. Surrogate grouping methods \cite{zhao2022towards} are also proposed to minimize the correlation between data features and model prediction, or directly predict surrogate demographic groups \cite{yan2020fair} and then perform group fairness algorithms \cite{sagawa2019distributionally, rahman2022fair}. Some debiasing methods also propose to identify the group disparities based on clustering information and upsample the minority groups to balance the distribution \cite{chaifairness, kimlearning}. 

\section{Theoretical Analysis}

\subsection{Problem Formulation}
Consider data $(x,y,s)$ with $n$ samples, where $x$ represents the non-sensitive features, $y$ represents the labels, and $s$ represents the sensitive attributes that determine the protected groups. Then, given $x$, we need to predict $y$ without the knowledge of $s$ while satisfying certain fairness criteria with respect to $s$. For example, the disparate impact requires $\hat{y} \perp s$ that the model prediction is independent of sensitive attributes, the equalized odd requires the independence conditional on ground truth label $\hat{y} \perp s|y$, and the Rawlsian fairness maximizes the utility of the worst-off group 
\begin{align}
h^* = \arg \max\limits_{h \in H} \min\limits_{s \in S} U_s(h),
\end{align}
 where $h$ denotes the hypothesis from the hypothesis class $H$, $U_s$ denotes the accuracy of the model on group $s$, and $s \in S$ where $S$ denotes all the possible groups, e.g., $S=[\text{white men}, \text{black women}, ...]$. Here, we consider the multiclass classification task, therefore $y \in \{M\}$ where $M$ denotes the number of classes.

\subsection{Definition of gradient}
Gradients provide both the training information about the learner model that indicates the algorithmic bias and the data information that indicates the data bias, therefore can help achieve fairness. First, there is a correlation between non-sensitive features $x$ and sensitive attributes $s$. For example, the prevalence of diseases is often affected by factors such as diets, which are affected by demographics such as religions and cultures. Second, as long as the model prediction accuracy is unfair across different groups, $U(h)$ is also correlated to sensitive attributes $s$. 
Nevertheless, we do not have true demographics $s$ as the label, so we do not know $s$ as an estimated function of $x$ and $U(h)$ a priori. Therefore, we propose to use gradients to represent demographic groups.

We consider a neural network $h$ parametrized by $\theta$ as $h(x;\theta)=\hat{y}$, where $\theta = (W,V)$. Thus, we have $W=(W_1, ..., W_d)^\top \in \R^{D \times M}$ as the weight of the last layer where $D$ denotes the dimensionality of the last latent representation. $V$ is the weight of all previous layers. 
$h(x;\theta)=\sigma(W \times z(x;V))$, where $\sigma(z)_j=e^{z_j}/\sum_{d=1}^D e^{z_d}$. The last-layer gradient w.r.t. the cross entropy loss is calculated as
\begin{align}
\frac{\partial}{\partial W} L(h(x;\theta),y) = z(x;V) \times (\hat{y}-y),
\end{align}
where
\begin{align}
L(h(x;\theta),y) &= -\sum_d y_d \cdot \log(h(x;\theta)) \\
&= \log \left( \sum_{d=1}^D e^{W_d \cdot z(x;V)} \right) - W_y \cdot z(x;V).
\end{align}

Note that $\hat{y}-y$ is the bias of the model prediction, which can have a positive/negative value. We define the undirected gradient $g \in \R^{D \times M}$ of the last layer of $h$ by
\begin{align}
\label{undirected_gradient}
g_{d,j} = z(x)_d |\hat{y_j}-y_j| = z_d U_j,
\end{align}
which is the multiplication of the latent representation of the non-sensitive feature and the prediction error (i.e., alternative to model accuracy $U$) of the label class. Here $y_j$ denotes the true value of the $j$-th class in the label.
We show that the undirected gradient to estimate the sensitive attributes is more accurate by both theoretical analysis and experimental verification.

\subsection{Relationship between gradients and sensitive demographics}
In this section, we discuss and theoretically analyze the effectiveness of gradients to represent sensitive demographics. We first generally demonstrate, through the lens of information theory and as articulated in Theorem \ref{theorem1}, that the distribution of gradients that combines the information of both input features and model error more closely aligns with sensitive demographic attributes as compared to input features, if input features are not perfect solutions for identifying demographics. Further, we explore under the condition of linear relationships, as outlined in Lemma \ref{lemma1}, that model gradients are more effective than input features in a larger correlation between input features and model prediction error.

\begin{theorem}
\label{theorem1}
The mutual dependence between gradients (the combination of input features and model error) and sensitive attributes is larger than the mutual dependence between input features and sensitive attributes. If we denote sensitive demographics as s, model prediction error as U, and input features as x, I(xU|s) > I(x|s).
\end{theorem}

\begin{proof}
\label{proof1main}
To calculate the mutual information $I(xU|s)$ between $xU$ and $s$, where $xU$ is the combination of both input features and model error as exemplified in Eq. \ref{undirected_gradient}, we have
\begin{align}
I(xU|s)
&=H(xU)-H(xU|s),
\end{align}
where 
$H(xU)$ is the marginal entropy of $xU$, $H(x|s)$ is the conditional entropy.
Similarly, we can calculate the mutual information $I(x|s)$ between $x$ and $s$ as 
\begin{align}
I(x|s)
&=H(x)-H(x|s).
\end{align}
Since $U$ is influenced by, or is a function of, $x$ and $s$ (different input features and demographics groups lead to different model error), we have $I(U;s|x)>0$. This means that $s$ introduces variations in $U$ that $x$ by itself cannot account for. This condition holds because $x$ and $s$ are not perfectly dependent on each other, that is, $I(U;s|x)>0$. Therefore, we have $H(U|x)-H(U|xs)>0$. 

Subtracting $I(xU|s)$ and $I(x|s)$, we have
\begin{align}
&\quad I(xU|s)-I(x|s)\nonumber \\
&=(H(xU)-H(x))-(H(xU|s)-H(x|s))\nonumber \\
&=H(U|x)-H(U|xs)\nonumber \\
&>0.
\end{align}

\end{proof}

Then, we demonstrate that, as a special case of Theorem \ref{theorem1}, the gradient of the last layer of a deep neural network can represent sensitive demographics better than input features alone. We first introduce Lemma \ref{lemma1}. Based on Lemma \ref{lemma1}, we show by Proposition \ref{pro1} that the last-layer gradient of a deep learning model is sufficient for the estimation of demographic groups. Intuitively, the last layer can capture high-level patterns in the data to make predictions, which is enough for representing sensitive demographics. Here, Pearson correlation in Lemma \ref{lemma1} is used as an example to illustrate our point. Specific nonlinear correlation is much harder to measure and analyze. Without further assumptions or knowledge about the data, there could be many possible nonlinear relationships (e.g., logarithmic, polynomial, exponential) and nonlinear correlation measurements (e.g., Hilbert-Schmidt Independence Criterion, Mutual Information, Maximal Information Coefficient).

\begin{lemma}
\label{lemma1}
$\frac{Corr(xU,s)}{Corr(x,s)}$ increases when $Corr(x, U)$ increases.
\end{lemma}

\begin{proof}
\label{proof2main}
Here, we assume the linearity of the data pattern for the simplicity of explanation. We hypothesize that there is a correlation between $x$ and $s$, and also a correlation between $U$ and $s$. For simplicity, we assume $s=ax+\epsilon_a$, where $a$ is the linear coefficient and $\epsilon_a$ is the noise term representing the part in $s$ that is uncorrelated to $x$. For simplicity, we can assume $\epsilon_a \sim N(\mu_a, \sigma_a^2)$. Similarly, we have $s=bU+\epsilon_b$ where $b$ is the linear coefficient and $\epsilon_b \sim N(\mu_b, \sigma_b^2)$. The two noise terms $\epsilon_a$ and $\epsilon_b$ are unknown and statistically independent. We assume that $x$ follows the standard normal distribution after preprocessing, thus $\mu_X=0, \sigma_X^2=1$. 

Note that we can rearrange the given equalities as follows:

$U=\frac{s-\epsilon_b}{b}$ and $x=\frac{s-\epsilon_a}{a} \rightarrow U=\frac{s-\epsilon_b}{b}=\frac{a}{b}x+\frac{\epsilon_a-\epsilon_b}{b}$ and $ 
xU=\frac{a}{b}x^2+\frac{\epsilon_a-\epsilon_b}{b}x.$

The covariance between $x$ and $s$ is
\begin{align}
Cov(x,s)&=Cov(x, ax+\epsilon_a) \\
&=a \cdot Cov(x, x)+Cov(x, \epsilon_a) \\
&=a,
\end{align}
where $Cov(x, \epsilon_a) = 0$ when the noise is independent of the feature. 

The correlation coefficient between $x$ and $s$ is
\begin{align}
Corr(x,s)&=\frac{Cov(x, s)}{\sqrt{Var(x)}\sqrt{Var(s)}} =\frac{a}{\sqrt{a^2+\sigma_a^2}}.
\end{align}

Similarly, we can calculate 
\begin{align}
Cov(x,U)&=E[x(\frac{a}{b}x+\frac{\epsilon_a-\epsilon_b}{b})]-0=\frac{a}{b}, \\
Corr(x, U)&=\frac{a}{\sqrt{a^2+\sigma^2_a+\sigma^2_b}}, \\
Var(s)&=Var(ax + \epsilon_a) = a^2 + \sigma_a^2, \\
Cov(U,s)&=Cov((1/b)(s - \epsilon_b), s) = \frac{a^2+\sigma_a^2}{b}, \\
Var(U)&=Var(\frac{a}{b}x)+Var(\frac{\epsilon_a-\epsilon_b}{b}) = \frac{a^2+\sigma^2_a+\sigma^2_b}{b^2}, \\
Corr(U,s)&=\frac{Cov(U,s)}{\sqrt{Var(U)}\sqrt{Var(s)}} =\frac{\sqrt{a^2+\sigma_a^2}}{\sqrt{a^2+\sigma_a^2+\sigma_b^2}}.
\end{align}
From observation, we find that when $\sigma^2_a+\sigma^2_b=0$, $Corr(x,U)=Corr(x,s)=Corr(U,s)=1$, regardless of the value of $a, b, \mu_a, \mu_b$. Since we regard $a$ as a constant that is not subject to change, we can conclude that $\sigma^2_a$ and $\sigma^2_b$ can directly determine the correlation.

We have the correlation between $xU$ and $s$ as
\begin{align}
Cov(xU, s) &= Cov\left(\frac{a}{b}x^2+\frac{\epsilon_a-\epsilon_b}{b}x, s\right) \\
&= \frac{a}{b}Cov(x^2, ax+\epsilon_a) \nonumber  + \frac{1}{b}Cov((\epsilon_a-\epsilon_b)x, ax+\epsilon_a) \\
&= \frac{a^2}{b}Cov(x^2, x) + \frac{1}{b}Cov((\epsilon_a-\epsilon_b)x, \epsilon_a) \nonumber \\
&\quad \   + \frac{1}{b}Cov((\epsilon_a-\epsilon_b)x, ax) \\
&= \frac{a^2}{b}(E[x^3] - E[x]E[x^2]) \nonumber  + \frac{aE(\epsilon_a-\epsilon_b)E(x^2)}{b} \\
&= \frac{a(\mu_a - \mu_b)}{b},
\end{align}
where, by the moments of standard normal distribution, $E(x^4) = 3$, $E(x^3) = 0$, and $E(x^2) = 1$. 

Then, we compute 
\begin{align}
    Var(xU)
    &=Var\left(\frac{a}{b}x^2+\frac{\epsilon_a-\epsilon_b}{b}x\right) \\
    &=\frac{a^2}{b^2}Var(x^2) +\frac{1}{b^2}\text{Var}(\epsilon_a - \epsilon_b) \text{Var}(x)\\
    &=\frac{a^2}{b^2}Var(x^2) \nonumber  +\frac{1}{b^2}(\text{Var}(\epsilon_a) + \text{Var}(\epsilon_b) - 2\text{Cov}(\epsilon_a, \epsilon_b)) \\
    &=\frac{2a^2+\sigma^2_a+\sigma^2_b}{b^2}
\end{align}
where $Var(x^2)=E(x^4)-E(x^2)^2=3-1=2$.

Therefore, the correlation coefficient between $xU$ and $s$ is
\begin{align}
Corr(xU,s)&=\frac{Cov(xU, s)}{\sqrt{Var(xU)}\sqrt{Var(s)}} \\
&=\frac{a(\mu_a - \mu_b)}{b\sqrt{a^2+\sigma_a^2}\sqrt{\frac{2a^2+\sigma^2_a+\sigma^2_b}{b^2}}},
\end{align}

Comparing $Corr(x,s)$ and $Corr(xU,s)$, we have
\begin{align}
Ratio=\frac{Corr(xU,s)}{Corr(x,s)}=\frac{\mu_a - \mu_b}{\sqrt{2a^2+\sigma^2_a+\sigma^2_b}}.
\end{align}

We can tell that both $Corr(x, U)$ and $Ratio$ are directly dependent on and are decreasing in $\sigma^2_a+\sigma^2_b$. Therefore, $\frac{Corr(xU,s)}{Corr(x,s)}$ increases in $Corr(x, U)$. 
\end{proof}

\begin{proposition}
\label{pro1}
The last-layer gradient of the deep learning prediction model can have a stronger correlation to sensitive attributes than non-sensitive input features. If we denote input features as $x$, model prediction error as $U$, last-layer representation as $z$, and sensitive attribute classes as $s$, we have $Corr(zU, s) > Corr(x, s)$.
\end{proposition}

\begin{proof}
\label{proof3main}
We can assume that the correlation between the last-layer representation and the label is larger than the correlation between the input features and the label, i.e., $Corr(z, s) > Corr(x, s)$. This assumption is very likely to hold in practice because the purpose of a neural network is to learn a representation to make it easier to predict the label. As long as the neural network is effectively learning representations, this assumption holds. According to Lemma \ref{lemma1}, we have
\begin{align}
\frac{Corr(zU, s)}{Corr(x, s)} > \frac{Corr(xU, s)}{Corr(x, s)},
\end{align}
which means $Corr(zU, s) > Corr(xU, s)$.
According to Theorem \ref{theorem1}, in general, $I(xU|s)>I(x|s)$. Under linear relationship, we have $Corr(xU,s)>Corr(x,s)$. In this case, $Corr(zU,s)>Corr(x,s)$.
\end{proof}

By Proposition \ref{pro1}, we show that the last-layer gradient of a deep learning model is more effective in identifying the demographic groups than the commonly used non-sensitive features. Based on this property, we develop an architecture in the next section to improve fairness without demographics.

\section{Methodology}
\label{Sec: method}

\begin{figure*}[t]
\centering
\includegraphics[width=0.7\linewidth]{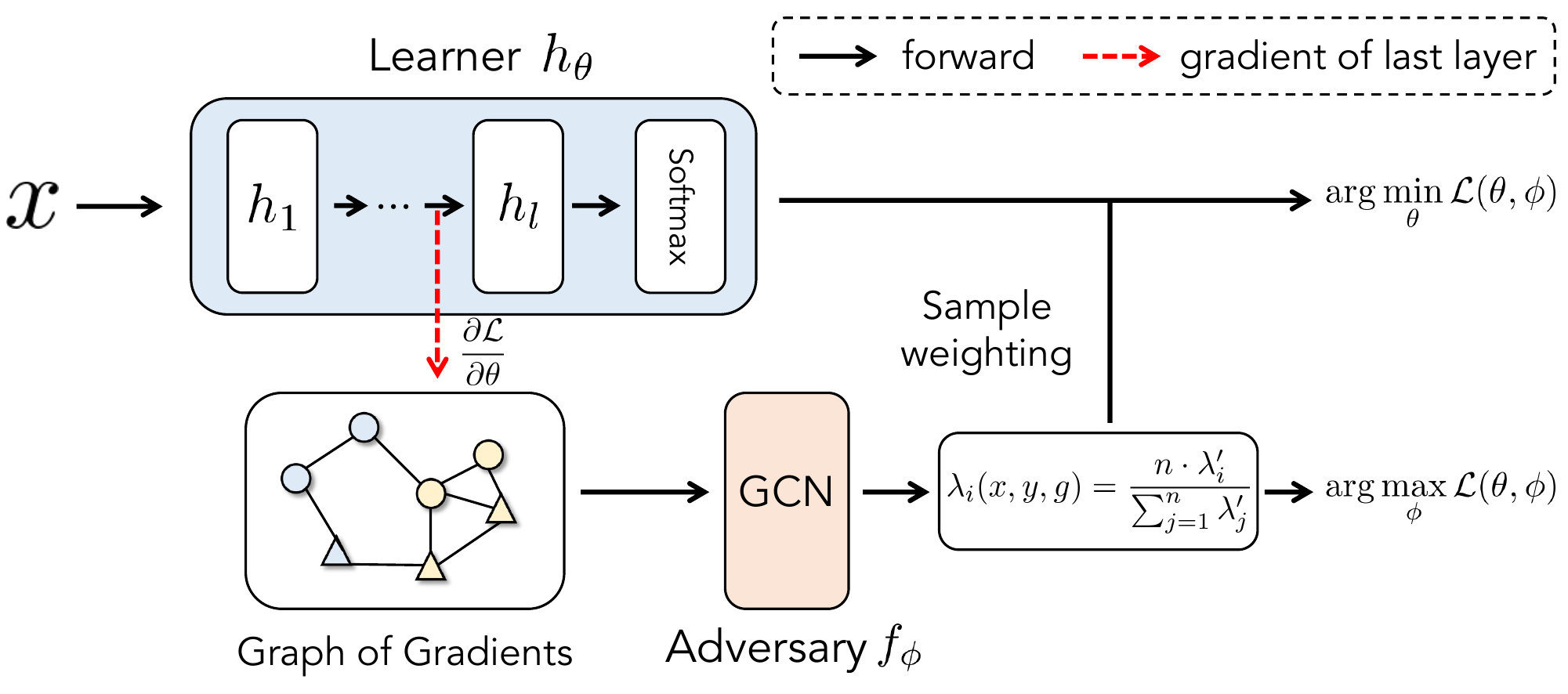}
\caption{The proposed Graph of Gradients algorithm framework.}
\label{fig:diagram}
\end{figure*}

In this section, we introduce the framework of our algorithm. 
Given the Rawlsian Max-Min fairness objective \cite{rawls2004theory} to maximize the utility of the worst-off group, how do we achieve it as a loss function? In the literature \cite{lahoti2020fairness}, it is proven that this fairness objective can be formulated as
\begin{align}
h^* &= \arg \mathop{min}\limits_{\theta} \mathop{max}_{\lambda} \mathcal{L} (\theta, \lambda) \\
&= \arg \mathop{min}_{\theta} \mathop{max}_{\lambda} \sum_{i=1}^n  \lambda_{s_i} L (h(x_i;\theta),y_i).
\end{align}
where $n$ denotes the number of samples. The optimal hypothesis $h^*$ that maximizes the utility is replaced by the model parameters $\theta$ minimizing the loss function, and the selection of demographic group $s$ that minimizes the utility is replaced by the reweighting of each demographic group with a learning weight $\lambda_s$. When the demographic groups are unknown to us, people learn to predict $s$ as a function of $x$ and $y$, or simply assign weights $\lambda_i$ to each sample based on features and labels.

The main difference between predicting $s$ and directly predicting $\lambda_i$ is whether a group partition is enforced. The former surrogate grouping methods mandate identifying the demographic groups and then assigning weights to different groups. The latter ones, such as the Adversarial Reweighted Learning (ARL) approach, directly assign weight generated by an adversary model to each sample. 

Since noisy data is an outlier to the true data distribution, it is difficult for a model to learn it correctly. To maximize the loss, the adversary will increase its sample weight to be unreasonably large. When generating weight for each sample instead of considering groups, this problem will be enlarged. In light of this issue, we propose to leverage the advantage of graph learning, which allows the interactions among connected neighbors to alleviate the problems with noisy labels. By learning from a graph, the weight generation of every sample will be dependent on the interdependencies between others with similar demographics. Assuming that most samples are not outliers, each outlier should be connected to some samples to avoid assigning large weights for outliers. Based on this notion, we propose to construct a graph where each sample is connected to the KK nearest samples. 

In this paper, we propose a novel learning framework to address the deficiencies of existing approaches. The overall framework is shown in Figure~\ref{fig:diagram}. Upon the observation that hard boundaries for group partition can be intractable with too many unknown demographics, we propose a method that mimics the grouping effect to alleviate the noises but does not specify boundaries. In detail, our method can reformulate Eq. 7 as follows
\begin{align}
\label{eq_first}
J(\theta, \phi)&=\min_{\theta} \max_{\phi} \sum_{i=1}^n  \lambda_i(x, y, g; \phi) \cdot L (h(x_i;\theta),y_i),
\end{align} 
where $\lambda$ is an adversary network powered by graph convolutional network and $g$ is the undirected last-layer gradients of all samples. $L$ is the cross-entropy loss as introduced in Eq.3. $x=[x_1, ..., x_n], y=[y_1, ..., y_n]$. For $g=[g_1, ..., g_n]$, the calculation of $g_i \in \R^{DM}$ for each sample $i \in \{ n \}$ is
\begin{align}
g_i &= \text{Flatten} \left( z(x_i) \times |h(x_i)-y_i| \right),
\end{align} 
where $h(x_i) \in \R^M$ denotes the prediction of the learner model and $y_i \in \R^M$ denotes the true label. $z(x_i) \in \R^D$ denotes the latent representation of the learner model before the last layer. 
The calculation of $\lambda$ is
\begin{align}
\lambda'(x,y,g) &= f_\phi(H, A) \in \R^n, \\
\lambda_i(x, y, g) &= \frac{n \cdot \lambda'_i}{\sum_{j=1}^n \lambda'_j},
\end{align} 
where $H=E_g+E_{x}$. Here, $E_g=gW^{(0)} \in \R^{n \times r}$ is the embedding of gradients, with $W^{(0)} \in \R^{DM \times r}$. $E_{x}=xW^{(1)} \in \R^{n \times r}$ is the data embedding, with $W^{(1)} \in \R^{t' \times r}$ where $t'$ denotes the number of dimensions of $x$. We use $\lambda_i \in \lambda$ to represent the weight for sample $i$, which is normalized according to Eq. 11. On the other hand, $A \in \R^{n \times n}$ is the adjacency matrix for constructing the Graph of Gradients (GoG), which is calculated as follows for a certain entry $A_{u, i}$
\begin{align}
A_{u, i} = \begin{cases}
1 & \text{if} \ dist(g_u, g_i) \geq dist(g_u)_{k} \\
0 & \text{if} \ dist(g_u, g_i) < dist(g_u)_{k}
\end{cases} ,
\end{align} 
where $dist(\cdot, \cdot)$ is the Euclidean distance between the gradients of two samples, $dist(\cdot)_k$ denotes the distance between the gradient of the sample and the gradient of its K-th nearest neighbors among the gradients of other samples. 

The calculation of $f_\phi$ can be represented as
\begin{align}
f(H, A) = \sigma(AHW^{(2)}) \in \R^{n \times 1}, 
\end{align} 
which is a one-layer graph convolutional network that takes each sample's data as input and outputs the sample weight. Therefore, $\phi=[W^{(0)}, W^{(1)}, W^{(2)}]$. To generate the weights, each sample, which is a node in the graph, aggregates the information of similar samples with K-nearest gradients. Here, $W^{(2)} \in \R^{r \times 1}$ allows the graph network to learn the importance of neighboring samples for aggregation, and $\sigma$ denotes the activation function. In practice, we implement the above equation via its improved version, such as
\begin{align}
\label{eq_last}
f(H, A) = \sigma(\hat{D}^{-\frac{1}{2}} \hat{A} \hat{D}^{-\frac{1}{2}} H W^{(2)}) \in \R^{n \times 1},
\end{align} 
where $\hat{A}=A+I$ where $I$ is the identity matrix and $\hat{D}$ is the diagonal node degree matrix of $\hat{A}$. 
The learning of GoG aims at enforcing a grouping mechanism based on the prior from the learner, without knowing the actual demographic groups a priori. When performing reweighting to minimize group disparity, samples that share similar gradients are connected and can interact with each other. 

The overall algorithm is detailed in Algorithm \ref{alg}. 
The overall complexity is linear. Therefore, it does not significantly increase the time complexity of the original learner model, which is also of an O(n) complexity. The K-nearest neighbor algorithm costs $O(dnB)$, where $d$ is the dimensionality of the input sequence, $B$ is the batch size, $n$ is the number of data samples. Because for each batch, the complexity is $O(dB^2)$, and there are $n/B$ batches in total. Then, the one-layer GCN has the complexity of $O(Knd + ndF)$ in total, where $|E|=Kn$ is the number of edges in the graph (according to the KNN algorithm), and $F$ is the dimensionality of the output for each node. Note that, only $n$ is a scalable variable of interest that grows with more data samples. Therefore, in total, the complexity of GoG is $O(dnB) + O(Knd + ndF)$, which is overall of O(n).

\begin{algorithm}
    \caption{Fairness without demographics with the graph of gradients.}
    \small
    \label{alg}
    \begin{algorithmic}[1]
        \REQUIRE The labelled data $x = \left[x_1, x_2, ..., x_t\right]$ and $y = \left[y_1, y_2, ..., y_t\right]$, the demographic sensitive attributes $a = \left[a_1, a_2, ..., a_t\right]$; the learner model $h_{\theta}$, the adversary model $\lambda_{\phi}$.
        \WHILE{not convergence}
            \FOR{i=0:n}
                \STATE Compute and record $h_{\theta}(x_i)$, $z(x_i)$ and $g_i$;
                \STATE Compute and record $L (h(x_i; \theta),y_i)$;
            \ENDFOR
            \STATE Compute $H$, $A$ and $\lambda_{\phi}(x, y, g)$ by Eqs.(\ref{eq_first}-\ref{eq_last});
            \STATE Compute $J(\theta, \phi)=\sum_{i=0}^n  \lambda(x, y, g; \phi)_i \cdot L (h(x_i; \theta),y_i)$;
            \STATE Fix $\theta$ and optimize $\phi$ by maximizing $J(\theta, \phi)$;
            \STATE Fix $\phi$ and optimize $\theta$ by minimizing $J(\theta, \phi)$;
        \ENDWHILE
        \RETURN $h_{\theta}$, which will be a fair model.
    \end{algorithmic}
\end{algorithm}

\section{Experiments}
\label{exp}
In this section, we conduct extensive experiments to verify the effectiveness of our method. Code is available\footnote{\url{https://github.com/yingtaoluo/Graph-of-Gradient.git}}.

\subsection{Experimental Setup}
We randomly divide each dataset by samples into the training, validation, and testing sets in a 0.75:0.1:0.15 ratio. 
We tune all the hyperparameters to obtain the optimal evaluation on the validation set for each model. The range of learning rate is \{1e-2, 3e-3, 1e-3\}, batch size is \{32, 64, 128\}, hidden dimension is \{16, 32, 64\}, dropout rate is \{0,1, 0.5\}, K is \{3, 10, 30\}. The training will stop if the accuracy of the worst group validation metrics does not increase in 20 epochs, and the test performance will be recorded. All results are averaged under five random seeds. 

We evaluate our method on three diverse datasets. COMPAS is a small-sized tabular dataset, BNP is a larger-sized tabular dataset with 20x more samples, and MIMIC-III is a large-scale sequential dataset with 100x more samples with noise and sparsity due to missing values. All datasets are \textbf{highly imbalanced}.
\begin{itemize}
\item \textbf{COMPAS Dataset} The Correctional Offender Management Profiling for Alternative Sanctions Dataset\footnote{\url{https://www.propublica.org/datastore/dataset/compas-recidivism-risk-score-data-and-analysis}} is a public criminology dataset containing the risk of recidivism. There are 7214 samples and 52 attributes. We use an MLP as the baseline. Sex, age, and race are protected attributes.

\item \textbf{BNP Dataset} The BNP Paribas Cardif Claims Management Dataset \footnote{\url{https://www.kaggle.com/c/bnp-paribas-cardif-claims-management/data}} is a public credit assessment dataset from BNP Paribas Cardif. There are 23 Categorical and 108 continuous attributes with 114321 data samples. We use a simple MLP as the baseline. Based on the $\chi^2$ statistics, v22, v56, v79, and v113 are selected as protected attributes. 

\item \textbf{MIMIC-III Dataset} The Medical Information Mart for Intensive Care database\footnote{\url{https://physionet.org/content/mimiciii/1.4/}} \cite{johnson2016mimic} has patients who stayed in the critical care units of the Beth Israel Deaconess Medical Center between 2001 and 2012. There are 53423 patients and 651048 diagnosis codes. The goal is to predict future diagnoses for multi-class classification. We follow the data pipeline in \cite{luo2022deep, luo2024fairness} and use LSTM as the baseline. Sex, age, and race are protected attributes.
\end{itemize}

\begin{table}[t]
\centering
  \caption{Performances of fairness algorithms, evaluated by Worst-Group AUC/Accuracy, Equalized Odds, and Overall AUC/Accuracy, averaged over five random seeds.}
  \label{tab:main1}
  \begin{adjustbox}{width=0.48\textwidth}
\begin{tabular}{ccccc}
\toprule
Dataset & Approach &  W. Acc($\uparrow$)         & E. Odds($\downarrow$)          & O. Acc($\uparrow$) \\
\midrule
COMPAS & Baseline       & 41.05$\pm$0.22\%       & 42.10$\pm$2.47\%        & 64.32$\pm$0.35\%  \\
COMPAS & DRO         & 44.20$\pm$0.42\%      & 31.66$\pm$2.98\%         & 56.13$\pm$0.45\%  \\
COMPAS & ARL         & 44.32$\pm$0.39\%      & 31.89$\pm$3.36\%         & 56.24$\pm$0.41\%  \\
COMPAS & FairCB         & 43.46$\pm$0.14\%      & 33.73$\pm$1.53\%         & $\textbf{57.32}$$\pm$0.11\%  \\
COMPAS & FairRF         & 44.24$\pm$0.18\%      & 31.54$\pm$1.88\%         & 56.67$\pm$0.21\%  \\
COMPAS & GoG        & $\textbf{44.85}$$\pm$0.37\%       & $\textbf{29.13}$$\pm$2.53\%        & 57.15$\pm$0.43\%       \\
\midrule
BNP & Baseline       & 48.25$\pm$0.20\%       & 47.73$\pm$0.84\%        & 73.28$\pm$0.32\%  \\
BNP & DRO         & 49.10$\pm$0.24\%      & 39.47$\pm$0.98\%         & 68.62$\pm$0.42\%  \\
BNP & ARL         & 48.92$\pm$0.20\%      & 39.14$\pm$1.36\%         & 68.80$\pm$0.38\%  \\
BNP & FairCB         & 48.00$\pm$0.31\%      & 42.56$\pm$1.55\%         & 68.35$\pm$0.14\%  \\
BNP & FairRF         & 48.79$\pm$0.24\%      & 40.89$\pm$1.68\%         & 68.96$\pm$0.31\%  \\
BNP & GoG        & $\textbf{50.27}$$\pm$0.21\%       & $\textbf{33.60}$$\pm$1.47\%        & $\textbf{69.55}$$\pm$0.51\%       \\
\midrule
MIMIC-III & Baseline       & 19.32$\pm$0.20\%       & 29.28$\pm$1.45\%        & 24.64$\pm$0.29\%  \\
MIMIC-III & DRO         & 19.68$\pm$0.14\%      & 25.75$\pm$2.19\%         & 22.62$\pm$0.22\%  \\
MIMIC-III & ARL         & 19.57$\pm$0.10\%      & 25.14$\pm$1.90\%         & 22.80$\pm$0.18\%  \\
MIMIC-III & FairCB         & 19.08$\pm$0.38\%      & 29.96$\pm$1.95\%         & 23.01$\pm$0.47\%  \\
MIMIC-III & FairRF         & 19.59$\pm$0.14\%      & 26.05$\pm$1.70\%         & 22.96$\pm$0.21\%  \\
MIMIC-III & GoG        & $\textbf{20.46}$$\pm$0.29\%       & $\textbf{18.99}$$\pm$1.72\%        & $\textbf{23.25}$$\pm$0.31\%       \\
\bottomrule
\end{tabular}
\end{adjustbox}
\end{table}

We compare our method that uses the graph of gradients with the following fairness-without-demographics models that use features for weight generation or clustering. 
\begin{itemize}
\item \textbf{DRO}: \cite{hashimoto2018fairness} A fair algorithm that uses $\chi^2$-divergence to discover and minimize the worst-case distribution repeatedly.
\item \textbf{ARL}: \cite{lahoti2020fairness} A fair algorithm that leverages computational identifiability to learn the demographics from features and labels for the Max-Min fairness objective.
\item \textbf{FairCB}: \cite{yan2020fair} A fair algorithm that performs cluster-based oversampling that identifies the subgroups and removes class bias in data.
\item \textbf{FairRF}: \cite{zhao2022towards} A fair algorithm that minimizes the correlation between data features and model predictions with importance weighting.
\end{itemize}

We evaluate fairness based on the following metrics. We are aware of the many possible fairness metrics for evaluation, such as disparate imapct. However, first, we are only interested in improving the model performance across protected groups instead of the results; second, in medical datasets, as pointed out by clinicians \cite{chow2012disparate}, since different patients do have different genotypes and phenotypes, it can be wrong to assume different patients to have similar probabilities for a certain disease. Therefore, we only use worst-case accuracy and Equalized odds in our experiments to avoid possible concerns. We also report the overall model performances to show the trade-off between fairness and performance.
\begin{itemize}
\item \textbf{Worst-Group AUC/Accuracy}: For COMPAS and BNP, we adopt the Area Under the Curve (AUC) for the worst subgroup. For MIMIC-III, we adopt the top-$20$ accuracy of the worst subpopulation. We abbreviate it as \textit{W. Acc}. 
\item \textbf{Equalized Odds}: We use the equalized odds (E. Odds), which requires the probability of instances with any two protected attributes $i$, $j$ being assigned to an outcome $k$ are equal, given the label.
In particular, for MIMIC-III, we cluster labels into eighteen diagnosis categories based on the ICD-9 code categories\footnote{\url{https://en.wikipedia.org/wiki/List_of_ICD-9_codes}} to calculate E. Odds. We calculate the total difference of all groups $S$ as
\begin{align}
\triangle_{EO}=\sum_{i, j}|E(\hat{y}|S=i,y=k)=E(\hat{y}|S=j,y=k)|.
\end{align}
\item \textbf{Overall AUC/Accuracy}: In addition to fairness metrics, we also present the AUC (for COMPAS and BNP) and/or Accuracy (for MIMIC-III) for all populations, abbr. \textit{O. Acc}.
\end{itemize}

\begin{figure*}[t]
\centering
\includegraphics[width=0.9\linewidth]{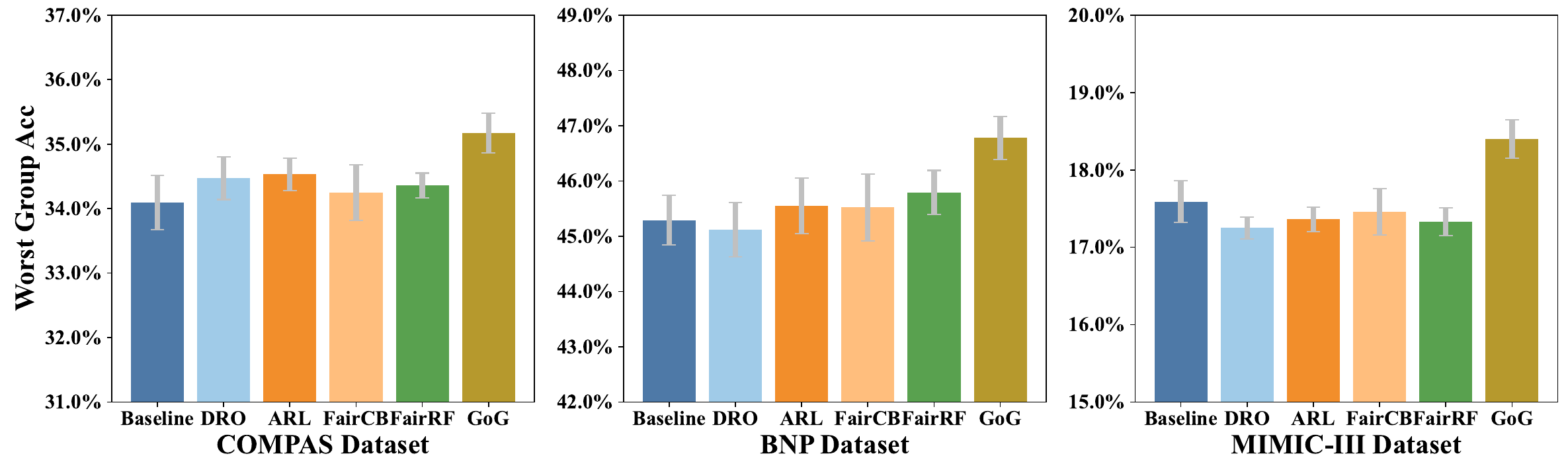}
\caption{Robustness study of different fairness algorithms across three datasets, evaluated by the Worst-Group AUC/Accuracy. }
\label{fig:robust}
\end{figure*}

\begin{table}[t]
\centering
  \caption{Ablation study of the graph and gradient components used in our algorithm. }
  \label{tab:main2}
\setlength{\tabcolsep}{1mm}{
\begin{adjustbox}{width=0.45\textwidth}
\begin{tabular}{ccccc}
\toprule
Dataset & Approach &  W. Acc($\uparrow$)         & E. Odds($\downarrow$)          & O. Acc($\uparrow$) \\
\midrule
COMPAS & GoG        & $\textbf{44.85}$$\pm$0.37\%       & $\textbf{29.13}$$\pm$2.53\%        & 57.15$\pm$0.40\%       \\
COMPAS & -Graph      & 44.38$\pm$0.35\%       & 29.14$\pm$2.50\%        & 56.22$\pm$0.43\%       \\
COMPAS & -Grad       & 44.36$\pm$0.40\%       & 29.52$\pm$2.23\%        & $\textbf{57.21}$$\pm$0.23\%       \\
\midrule
BNP & GoG        & $\textbf{50.27}$$\pm$0.21\%       & $\textbf{33.60}$$\pm$1.47\%        & $\textbf{69.85}$$\pm$0.51\%       \\
BNP & -Graph       & 49.33$\pm$0.13\%       & 38.35$\pm$1.15\%        & 68.54$\pm$0.14\%  \\
BNP & -Grad         & 49.98$\pm$0.09\%      & 35.29$\pm$1.19\%         & 69.81$\pm$0.06\%  \\
\midrule
MIMIC-III & GoG        & $\textbf{20.46}$$\pm$0.29\%       & $\textbf{18.99}$$\pm$1.72\%        & $\textbf{23.25}$$\pm$0.31\%       \\
MIMIC-III & -Graph       & 19.98$\pm$0.25\%       & 24.27$\pm$0.56\%        & 22.82$\pm$0.27\%  \\
MIMIC-III & -Grad         & 20.31$\pm$0.21\%      & 20.08$\pm$0.42\%         & 23.14$\pm$0.26\%  \\
\bottomrule
\end{tabular}
\end{adjustbox}}
\end{table}

\subsection{Performance Comparison}
We conduct performance comparison among baseline models and several fairness algorithms on the three datasets, with results shown in Table. \ref{tab:main1}.
Overall, Table \ref{tab:main1} demonstrates that our algorithm significantly enhances the fairness of machine learning models, with the \textit{W.Acc} and \textit{E. Odds} outperforming other fairness algorithms. Furthermore, while all fairness algorithms result in reduced accuracy, our algorithm surpasses them in overall accuracy, indicating a superior balance between fairness and accuracy. GoG demonstrates to work well with different deep learning baselines, such as MLP for COMPAS and BNP and LSTM for MIMIC-III in our cases.

Notably, we observe that our method's improvements are more significant in complex datasets. MIMIC-III and BNP datasets contain more complex structures, more demographic subgroups, attributes, and data samples than the COMPAS dataset and are likely to have more outliers, as mentioned in their sources. In particular, the noises in data can damage the fairness performance severely. 

In addition, we provide an empirical time cost comparison. With an A100 80G graphical card, on MIMIC-III dataset, running the learner costs 2.7 seconds per epoch, while running the GoG costs 6.4 seconds per epoch; on the COMPAS dataset, it is 0.6 seconds against 1.6 seconds; on BNP dataset, it is 1.5 seconds against 3.3 seconds. The additional training time increases at most by twice, and the inference speed is the same, since the weighting provided by GoG only happens during training. 

\subsection{Robustness Study}
To assess our model's robustness to noise, we create noisy datasets by altering the labels of 10\% of samples in the original datasets, treating these false labels as true during evaluation. Samples with wrong labels are harder to learn due to the deviation from the data distribution, which forces the fairness algorithm to focus blindly on the samples that cannot improve the model fairness for real samples. 
We train various models on the noisy datasets to evaluate fairness, with results shown in Figure. \ref{fig:robust}. Some baselines can occasionally perform even worse than the baseline, which shows that noisy labels can severely damage existing fairness algorithms. Our model's fairness metrics do not degrade as quickly as other fairness algorithms under noise, showing great robustness.

\begin{table}[t]
\centering
  \caption{Additional study of hyperparameter K for GoG. }
  \label{tab:k}
\setlength{\tabcolsep}{1mm}{
\begin{adjustbox}{width=0.40\textwidth}
\begin{tabular}{ccccc}
\toprule
Dataset & K &  W. Acc($\uparrow$)         & E. Odds($\downarrow$)          & O. Acc($\uparrow$) \\
\midrule
COMPAS & 3        & $\textbf{44.85}$$\pm$0.37\%       & 29.13$\pm$2.53\%        & $\textbf{57.15}$$\pm$0.40\%       \\
COMPAS & 10      & 44.63$\pm$0.32\%       & 29.10$\pm$2.37\%        & 57.12$\pm$0.35\%       \\
COMPAS & 30       & 44.67$\pm$0.29\%       & $\textbf{29.17}$$\pm$2.51\%        & 57.12$\pm$0.37\%       \\
\midrule
BNP & 3        & $\textbf{50.27}$$\pm$0.21\%       & 33.60$\pm$1.47\%        & $\textbf{69.85}$$\pm$0.51\%       \\
BNP & 10       & 50.02$\pm$0.18\%       & 33.46$\pm$1.52\%        & 69.69$\pm$0.66\%  \\
BNP & 30         & 50.22$\pm$0.19\%      & $\textbf{33.68}$$\pm$1.86\%         & 69.68$\pm$0.56\%  \\
\midrule
MIMIC-III & 3        & $\textbf{20.46}$$\pm$0.29\%       & 18.99$\pm$1.72\%        & $\textbf{23.25}$$\pm$0.31\%       \\
MIMIC-III & 10       & 20.35$\pm$0.28\%       & $\textbf{19.00}$$\pm$1.67\%        & 23.21$\pm$0.33\%  \\
MIMIC-III & 30         & 20.42$\pm$0.29\%      & 18.92$\pm$1.63\%         & 23.22$\pm$0.25\%  \\
\bottomrule
\end{tabular}
\end{adjustbox}}
\end{table}

\begin{table}[t]
\centering
  \caption{Additional study of last versus first layer for GoG. }
  \label{tab:layer}
\setlength{\tabcolsep}{1mm}{
\begin{adjustbox}{width=0.42\textwidth}
\begin{tabular}{ccccc}
\toprule
Dataset & Layer &  W. Acc($\uparrow$)         & E. Odds($\downarrow$)          & O. Acc($\uparrow$) \\
\midrule
COMPAS & Last       & $\textbf{44.85}$$\pm$0.37\%       & $\textbf{29.13}$$\pm$2.53\%        & $\textbf{57.15}$$\pm$0.40\%       \\
COMPAS & First       & 44.56$\pm$0.25\%       & 30.35$\pm$2.78\%        & 57.10$\pm$0.66\%       \\
\midrule
BNP & Last        & $\textbf{50.27}$$\pm$0.21\%       & $\textbf{33.60}$$\pm$1.47\%        & $\textbf{69.85}$$\pm$0.51\%       \\
BNP & First        & 49.89$\pm$0.25\%       & 34.92$\pm$1.32\%        & 69.31$\pm$0.61\%  \\
\midrule
MIMIC-III & Last        & $\textbf{20.46}$$\pm$0.29\%       & $\textbf{18.99}$$\pm$1.72\%        & $\textbf{23.25}$$\pm$0.31\%       \\
MIMIC-III & First        & 20.39$\pm$0.28\%       & 19.20$\pm$1.58\%        & 23.16$\pm$0.30\%  \\
\bottomrule
\end{tabular}
\end{adjustbox}}
\end{table}

\subsection{Ablation Study}
We conduct an ablation study to investigate the components of our method, which consists of 1) the constructing and learning the graph of gradients (GoG), and 2) the use of gradients as a substitute for features to represent unknown demographics. We assess the importance of each component by reporting the fairness metrics when a particular part is removed. The ablation results, shown in Table \ref{tab:main2}, indicate that both components contribute positively, as the model without graph or gradients still outperforms baselines. Furthermore, we observe that graph learning (-Grad) plays a more critical role, particularly for the more complex BNP and MIMIC-III datasets. GoG contributes to the improved trade-off between fairness and accuracy.
In addition, we also conduct a hyperparameter sensitivity study for $K$ in Table \ref{tab:k}, which shows that $K$ does not affect performance. When we use first-layer gradient instead of the last-year gradient, as shown in Table \ref{tab:layer}, the performance decreases slightly. This demonstrates the effectiveness of last-layer gradient. 

\subsection{Interpretability Study}
Figure \ref{fig:tsne} shows that while it is almost impossible to separate people by race based on input features shown in Figure \ref{fig:tsne}a, we find that the least populated subpopulation group, colored by green, is distributed closely on the graph of last-layer gradient shown in Figure \ref{fig:tsne}b. Last-layer gradient is distributed more evenly in the space, which helps the classification of groups. This case study supports our claim that gradient is a better proxy for true demographics.

\begin{figure}[t]
\centering
\includegraphics[width=0.95\linewidth]{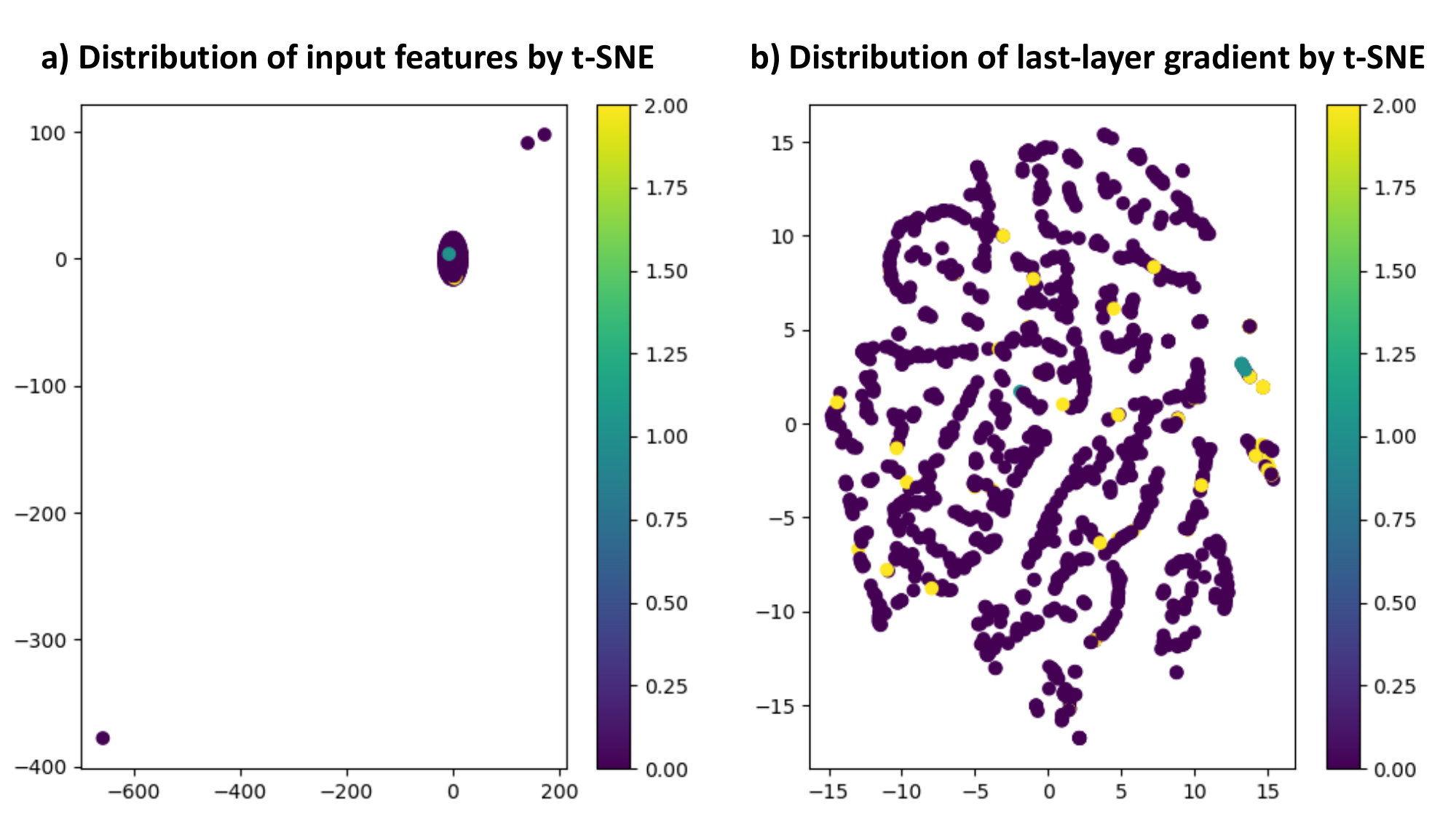}
\caption{The distribution of input features and last-layer gradients upon different race groups for MIMIC-III dataset (Purple: 'White', Yellow: 'BLACK', Green: 'HISPANIC', Other races are not considered for convenience).}
\label{fig:tsne}
\end{figure}

\section{Conclusion}

In this study, we address the fairness concerns in machine learning algorithms, focusing on the challenges posed by the complex interplay of demographic variables and regulatory constraints, which often render demographic information unknown. We present a novel approach to tackle the limitations of existing methods. We demonstrate that model gradient can better identify unknown demographics and propose the graph of gradients by connecting each sample to its K-nearest neighbors to identify demographic groups and generate sample weights. Experimental results reveal that our method significantly enhances the algorithmic fairness, exhibiting notable robustness to noise that can be common in the real world. Since GoG depends on the gradients of learner, there might be a limitation if the learner cannot effectively capture the data pattern, e.g., data is not big enough and model is not strong enough.

\bibliographystyle{ACM-Reference-Format}
\bibliography{ref}


\begin{thebibliography}{38}


\ifx \showCODEN    \undefined \def \showCODEN     #1{\unskip}     \fi
\ifx \showDOI      \undefined \def \showDOI       #1{#1}\fi
\ifx \showISBNx    \undefined \def \showISBNx     #1{\unskip}     \fi
\ifx \showISBNxiii \undefined \def \showISBNxiii  #1{\unskip}     \fi
\ifx \showISSN     \undefined \def \showISSN      #1{\unskip}     \fi
\ifx \showLCCN     \undefined \def \showLCCN      #1{\unskip}     \fi
\ifx \shownote     \undefined \def \shownote      #1{#1}          \fi
\ifx \showarticletitle \undefined \def \showarticletitle #1{#1}   \fi
\ifx \showURL      \undefined \def \showURL       {\relax}        \fi
\providecommand\bibfield[2]{#2}
\providecommand\bibinfo[2]{#2}
\providecommand\natexlab[1]{#1}
\providecommand\showeprint[2][]{arXiv:#2}

\bibitem[\protect\citeauthoryear{Celis, Huang, Keswani, and Vishnoi}{Celis et~al\mbox{.}}{2021a}]%
        {celis2021fair}
\bibfield{author}{\bibinfo{person}{L~Elisa Celis}, \bibinfo{person}{Lingxiao Huang}, \bibinfo{person}{Vijay Keswani}, {and} \bibinfo{person}{Nisheeth~K Vishnoi}.} \bibinfo{year}{2021}\natexlab{a}.
\newblock \showarticletitle{Fair classification with noisy protected attributes: A framework with provable guarantees}. In \bibinfo{booktitle}{\emph{International Conference on Machine Learning}}. PMLR, \bibinfo{pages}{1349--1361}.
\newblock


\bibitem[\protect\citeauthoryear{Celis, Mehrotra, and Vishnoi}{Celis et~al\mbox{.}}{2021b}]%
        {celis2021fairb}
\bibfield{author}{\bibinfo{person}{L~Elisa Celis}, \bibinfo{person}{Anay Mehrotra}, {and} \bibinfo{person}{Nisheeth Vishnoi}.} \bibinfo{year}{2021}\natexlab{b}.
\newblock \showarticletitle{Fair classification with adversarial perturbations}.
\newblock \bibinfo{journal}{\emph{Advances in Neural Information Processing Systems}}  \bibinfo{volume}{34} (\bibinfo{year}{2021}), \bibinfo{pages}{8158--8171}.
\newblock


\bibitem[\protect\citeauthoryear{Chai, Jang, and Wang}{Chai et~al\mbox{.}}{2022}]%
        {chaifairness}
\bibfield{author}{\bibinfo{person}{Junyi Chai}, \bibinfo{person}{Taeuk Jang}, {and} \bibinfo{person}{Xiaoqian Wang}.} \bibinfo{year}{2022}\natexlab{}.
\newblock \showarticletitle{Fairness without Demographics through Knowledge Distillation}. In \bibinfo{booktitle}{\emph{Advances in Neural Information Processing Systems}}.
\newblock


\bibitem[\protect\citeauthoryear{Chai and Wang}{Chai and Wang}{2022}]%
        {chai2022fairness}
\bibfield{author}{\bibinfo{person}{Junyi Chai} {and} \bibinfo{person}{Xiaoqian Wang}.} \bibinfo{year}{2022}\natexlab{}.
\newblock \showarticletitle{Fairness with adaptive weights}. In \bibinfo{booktitle}{\emph{International Conference on Machine Learning}}. PMLR, \bibinfo{pages}{2853--2866}.
\newblock


\bibitem[\protect\citeauthoryear{Chen, Johansson, and Sontag}{Chen et~al\mbox{.}}{2018}]%
        {chen2018my}
\bibfield{author}{\bibinfo{person}{Irene Chen}, \bibinfo{person}{Fredrik~D Johansson}, {and} \bibinfo{person}{David Sontag}.} \bibinfo{year}{2018}\natexlab{}.
\newblock \showarticletitle{Why is my classifier discriminatory?}
\newblock \bibinfo{journal}{\emph{Advances in neural information processing systems}}  \bibinfo{volume}{31} (\bibinfo{year}{2018}).
\newblock


\bibitem[\protect\citeauthoryear{Chouldechova}{Chouldechova}{2017}]%
        {chouldechova2017fair}
\bibfield{author}{\bibinfo{person}{Alexandra Chouldechova}.} \bibinfo{year}{2017}\natexlab{}.
\newblock \showarticletitle{Fair prediction with disparate impact: A study of bias in recidivism prediction instruments}.
\newblock \bibinfo{journal}{\emph{Big data}} \bibinfo{volume}{5}, \bibinfo{number}{2} (\bibinfo{year}{2017}), \bibinfo{pages}{153--163}.
\newblock


\bibitem[\protect\citeauthoryear{Chow, Foster, Gonzalez, and McIver}{Chow et~al\mbox{.}}{2012}]%
        {chow2012disparate}
\bibfield{author}{\bibinfo{person}{Edward~A Chow}, \bibinfo{person}{Henry Foster}, \bibinfo{person}{Victor Gonzalez}, {and} \bibinfo{person}{LaShawn McIver}.} \bibinfo{year}{2012}\natexlab{}.
\newblock \showarticletitle{The disparate impact of diabetes on racial/ethnic minority populations}.
\newblock \bibinfo{journal}{\emph{Clinical Diabetes}} \bibinfo{volume}{30}, \bibinfo{number}{3} (\bibinfo{year}{2012}), \bibinfo{pages}{130--133}.
\newblock


\bibitem[\protect\citeauthoryear{Creager, Jacobsen, and Zemel}{Creager et~al\mbox{.}}{2021}]%
        {creager2021environment}
\bibfield{author}{\bibinfo{person}{Elliot Creager}, \bibinfo{person}{J{\"o}rn-Henrik Jacobsen}, {and} \bibinfo{person}{Richard Zemel}.} \bibinfo{year}{2021}\natexlab{}.
\newblock \showarticletitle{Environment inference for invariant learning}. In \bibinfo{booktitle}{\emph{International Conference on Machine Learning}}. PMLR, \bibinfo{pages}{2189--2200}.
\newblock


\bibitem[\protect\citeauthoryear{Dong, Kose, Shen, and Li}{Dong et~al\mbox{.}}{2023a}]%
        {dong2023fairness}
\bibfield{author}{\bibinfo{person}{Yushun Dong}, \bibinfo{person}{Oyku~Deniz Kose}, \bibinfo{person}{Yanning Shen}, {and} \bibinfo{person}{Jundong Li}.} \bibinfo{year}{2023}\natexlab{a}.
\newblock \showarticletitle{Fairness in graph machine learning: Recent advances and future prospectives}. In \bibinfo{booktitle}{\emph{Proceedings of the 29th ACM SIGKDD Conference on Knowledge Discovery and Data Mining}}. \bibinfo{pages}{5794--5795}.
\newblock


\bibitem[\protect\citeauthoryear{Dong, Liu, Jalaian, and Li}{Dong et~al\mbox{.}}{2022}]%
        {dong2022edits}
\bibfield{author}{\bibinfo{person}{Yushun Dong}, \bibinfo{person}{Ninghao Liu}, \bibinfo{person}{Brian Jalaian}, {and} \bibinfo{person}{Jundong Li}.} \bibinfo{year}{2022}\natexlab{}.
\newblock \showarticletitle{Edits: Modeling and mitigating data bias for graph neural networks}. In \bibinfo{booktitle}{\emph{Proceedings of the ACM web conference 2022}}. \bibinfo{pages}{1259--1269}.
\newblock


\bibitem[\protect\citeauthoryear{Dong, Wang, Ma, Liu, and Li}{Dong et~al\mbox{.}}{2023b}]%
        {dong2023interpreting}
\bibfield{author}{\bibinfo{person}{Yushun Dong}, \bibinfo{person}{Song Wang}, \bibinfo{person}{Jing Ma}, \bibinfo{person}{Ninghao Liu}, {and} \bibinfo{person}{Jundong Li}.} \bibinfo{year}{2023}\natexlab{b}.
\newblock \showarticletitle{Interpreting unfairness in graph neural networks via training node attribution}. In \bibinfo{booktitle}{\emph{Proceedings of the AAAI Conference on Artificial Intelligence}}, Vol.~\bibinfo{volume}{37}. \bibinfo{pages}{7441--7449}.
\newblock


\bibitem[\protect\citeauthoryear{Du, Mukherjee, Wang, Tang, Awadallah, and Hu}{Du et~al\mbox{.}}{2021}]%
        {du2021fairness}
\bibfield{author}{\bibinfo{person}{Mengnan Du}, \bibinfo{person}{Subhabrata Mukherjee}, \bibinfo{person}{Guanchu Wang}, \bibinfo{person}{Ruixiang Tang}, \bibinfo{person}{Ahmed Awadallah}, {and} \bibinfo{person}{Xia Hu}.} \bibinfo{year}{2021}\natexlab{}.
\newblock \showarticletitle{Fairness via representation neutralization}.
\newblock \bibinfo{journal}{\emph{Advances in Neural Information Processing Systems}}  \bibinfo{volume}{34} (\bibinfo{year}{2021}), \bibinfo{pages}{12091--12103}.
\newblock


\bibitem[\protect\citeauthoryear{Gianfrancesco, Tamang, Yazdany, and Schmajuk}{Gianfrancesco et~al\mbox{.}}{2018}]%
        {gianfrancesco2018potential}
\bibfield{author}{\bibinfo{person}{Milena~A Gianfrancesco}, \bibinfo{person}{Suzanne Tamang}, \bibinfo{person}{Jinoos Yazdany}, {and} \bibinfo{person}{Gabriela Schmajuk}.} \bibinfo{year}{2018}\natexlab{}.
\newblock \showarticletitle{Potential biases in machine learning algorithms using electronic health record data}.
\newblock \bibinfo{journal}{\emph{JAMA internal medicine}} \bibinfo{volume}{178}, \bibinfo{number}{11} (\bibinfo{year}{2018}), \bibinfo{pages}{1544--1547}.
\newblock


\bibitem[\protect\citeauthoryear{Giguere, Metevier, Castro~da Silva, Brun, Thomas, and Niekum}{Giguere et~al\mbox{.}}{2022}]%
        {giguere2022fairness}
\bibfield{author}{\bibinfo{person}{Stephen Giguere}, \bibinfo{person}{Blossom Metevier}, \bibinfo{person}{Bruno Castro~da Silva}, \bibinfo{person}{Yuriy Brun}, \bibinfo{person}{Philip Thomas}, {and} \bibinfo{person}{Scott Niekum}.} \bibinfo{year}{2022}\natexlab{}.
\newblock \showarticletitle{Fairness guarantees under demographic shift}. In \bibinfo{booktitle}{\emph{International Conference on Learning Representations}}.
\newblock


\bibitem[\protect\citeauthoryear{Grari, Lamprier, and Detyniecki}{Grari et~al\mbox{.}}{2021}]%
        {grari2021fairness}
\bibfield{author}{\bibinfo{person}{Vincent Grari}, \bibinfo{person}{Sylvain Lamprier}, {and} \bibinfo{person}{Marcin Detyniecki}.} \bibinfo{year}{2021}\natexlab{}.
\newblock \showarticletitle{Fairness without the sensitive attribute via Causal Variational Autoencoder}.
\newblock \bibinfo{journal}{\emph{arXiv preprint arXiv:2109.04999}} (\bibinfo{year}{2021}).
\newblock


\bibitem[\protect\citeauthoryear{Hajian, Bonchi, and Castillo}{Hajian et~al\mbox{.}}{2016}]%
        {hajian2016algorithmic}
\bibfield{author}{\bibinfo{person}{Sara Hajian}, \bibinfo{person}{Francesco Bonchi}, {and} \bibinfo{person}{Carlos Castillo}.} \bibinfo{year}{2016}\natexlab{}.
\newblock \showarticletitle{Algorithmic bias: From discrimination discovery to fairness-aware data mining}. In \bibinfo{booktitle}{\emph{Proceedings of the 22nd ACM SIGKDD international conference on knowledge discovery and data mining}}. \bibinfo{pages}{2125--2126}.
\newblock


\bibitem[\protect\citeauthoryear{Hardt, Price, and Srebro}{Hardt et~al\mbox{.}}{2016}]%
        {hardt2016equality}
\bibfield{author}{\bibinfo{person}{Moritz Hardt}, \bibinfo{person}{Eric Price}, {and} \bibinfo{person}{Nati Srebro}.} \bibinfo{year}{2016}\natexlab{}.
\newblock \showarticletitle{Equality of opportunity in supervised learning}.
\newblock \bibinfo{journal}{\emph{Advances in neural information processing systems}}  \bibinfo{volume}{29} (\bibinfo{year}{2016}).
\newblock


\bibitem[\protect\citeauthoryear{Hashimoto, Srivastava, Namkoong, and Liang}{Hashimoto et~al\mbox{.}}{2018}]%
        {hashimoto2018fairness}
\bibfield{author}{\bibinfo{person}{Tatsunori Hashimoto}, \bibinfo{person}{Megha Srivastava}, \bibinfo{person}{Hongseok Namkoong}, {and} \bibinfo{person}{Percy Liang}.} \bibinfo{year}{2018}\natexlab{}.
\newblock \showarticletitle{Fairness without demographics in repeated loss minimization}. In \bibinfo{booktitle}{\emph{International Conference on Machine Learning}}. PMLR, \bibinfo{pages}{1929--1938}.
\newblock


\bibitem[\protect\citeauthoryear{Iosifidis and Ntoutsi}{Iosifidis and Ntoutsi}{2019}]%
        {iosifidis2019adafair}
\bibfield{author}{\bibinfo{person}{Vasileios Iosifidis} {and} \bibinfo{person}{Eirini Ntoutsi}.} \bibinfo{year}{2019}\natexlab{}.
\newblock \showarticletitle{Adafair: Cumulative fairness adaptive boosting}. In \bibinfo{booktitle}{\emph{Proceedings of the 28th ACM international conference on information and knowledge management}}. \bibinfo{pages}{781--790}.
\newblock


\bibitem[\protect\citeauthoryear{Jang, Shi, and Wang}{Jang et~al\mbox{.}}{2022}]%
        {jang2022group}
\bibfield{author}{\bibinfo{person}{Taeuk Jang}, \bibinfo{person}{Pengyi Shi}, {and} \bibinfo{person}{Xiaoqian Wang}.} \bibinfo{year}{2022}\natexlab{}.
\newblock \showarticletitle{Group-aware threshold adaptation for fair classification}. In \bibinfo{booktitle}{\emph{Proceedings of the AAAI Conference on Artificial Intelligence}}, Vol.~\bibinfo{volume}{36}. \bibinfo{pages}{6988--6995}.
\newblock


\bibitem[\protect\citeauthoryear{Jang, Zheng, and Wang}{Jang et~al\mbox{.}}{2021}]%
        {jang2021constructing}
\bibfield{author}{\bibinfo{person}{Taeuk Jang}, \bibinfo{person}{Feng Zheng}, {and} \bibinfo{person}{Xiaoqian Wang}.} \bibinfo{year}{2021}\natexlab{}.
\newblock \showarticletitle{Constructing a fair classifier with generated fair data}. In \bibinfo{booktitle}{\emph{Proceedings of the AAAI Conference on Artificial Intelligence}}, Vol.~\bibinfo{volume}{35}. \bibinfo{pages}{7908--7916}.
\newblock


\bibitem[\protect\citeauthoryear{Jiang, Han, Fan, Yang, Mostafavi, and Hu}{Jiang et~al\mbox{.}}{2022}]%
        {jiang2022generalized}
\bibfield{author}{\bibinfo{person}{Zhimeng Jiang}, \bibinfo{person}{Xiaotian Han}, \bibinfo{person}{Chao Fan}, \bibinfo{person}{Fan Yang}, \bibinfo{person}{Ali Mostafavi}, {and} \bibinfo{person}{Xia Hu}.} \bibinfo{year}{2022}\natexlab{}.
\newblock \showarticletitle{Generalized demographic parity for group fairness}. In \bibinfo{booktitle}{\emph{International Conference on Learning Representations}}.
\newblock


\bibitem[\protect\citeauthoryear{Johnson, Pollard, Shen, Li-Wei, Feng, Ghassemi, Moody, Szolovits, Celi, and Mark}{Johnson et~al\mbox{.}}{2016}]%
        {johnson2016mimic}
\bibfield{author}{\bibinfo{person}{Alistair~EW Johnson}, \bibinfo{person}{Tom~J Pollard}, \bibinfo{person}{Lu Shen}, \bibinfo{person}{H~Lehman Li-Wei}, \bibinfo{person}{Mengling Feng}, \bibinfo{person}{Mohammad Ghassemi}, \bibinfo{person}{Benjamin Moody}, \bibinfo{person}{Peter Szolovits}, \bibinfo{person}{Leo~Anthony Celi}, {and} \bibinfo{person}{Roger~G Mark}.} \bibinfo{year}{2016}\natexlab{}.
\newblock \showarticletitle{MIMIC-III, a freely accessible critical care database}.
\newblock \bibinfo{journal}{\emph{Scientific Data}} \bibinfo{volume}{3}, \bibinfo{number}{1} (\bibinfo{year}{2016}), \bibinfo{pages}{1--9}.
\newblock


\bibitem[\protect\citeauthoryear{Kim, Hwang, Ahn, Park, and Kwak}{Kim et~al\mbox{.}}{2022}]%
        {kimlearning}
\bibfield{author}{\bibinfo{person}{Nayeong Kim}, \bibinfo{person}{Sehyun Hwang}, \bibinfo{person}{Sungsoo Ahn}, \bibinfo{person}{Jaesik Park}, {and} \bibinfo{person}{Suha Kwak}.} \bibinfo{year}{2022}\natexlab{}.
\newblock \showarticletitle{Learning Debiased Classifier with Biased Committee}. In \bibinfo{booktitle}{\emph{Advances in Neural Information Processing Systems}}.
\newblock


\bibitem[\protect\citeauthoryear{Lahoti, Beutel, Chen, Lee, Prost, Thain, Wang, and Chi}{Lahoti et~al\mbox{.}}{2020}]%
        {lahoti2020fairness}
\bibfield{author}{\bibinfo{person}{Preethi Lahoti}, \bibinfo{person}{Alex Beutel}, \bibinfo{person}{Jilin Chen}, \bibinfo{person}{Kang Lee}, \bibinfo{person}{Flavien Prost}, \bibinfo{person}{Nithum Thain}, \bibinfo{person}{Xuezhi Wang}, {and} \bibinfo{person}{Ed Chi}.} \bibinfo{year}{2020}\natexlab{}.
\newblock \showarticletitle{Fairness without demographics through adversarially reweighted learning}.
\newblock \bibinfo{journal}{\emph{Advances in neural information processing systems}}  \bibinfo{volume}{33} (\bibinfo{year}{2020}), \bibinfo{pages}{728--740}.
\newblock


\bibitem[\protect\citeauthoryear{Li, Dong, Liu, and Yu}{Li et~al\mbox{.}}{2024}]%
        {li2024rethinking}
\bibfield{author}{\bibinfo{person}{Zhixun Li}, \bibinfo{person}{Yushun Dong}, \bibinfo{person}{Qiang Liu}, {and} \bibinfo{person}{Jeffrey~Xu Yu}.} \bibinfo{year}{2024}\natexlab{}.
\newblock \showarticletitle{Rethinking fair graph neural networks from re-balancing}. In \bibinfo{booktitle}{\emph{Proceedings of the 30th ACM SIGKDD Conference on Knowledge Discovery and Data Mining}}. \bibinfo{pages}{1736--1745}.
\newblock


\bibitem[\protect\citeauthoryear{Luo, Li, Liu, and Zhu}{Luo et~al\mbox{.}}{2024}]%
        {luo2024fairness}
\bibfield{author}{\bibinfo{person}{Yingtao Luo}, \bibinfo{person}{Zhixun Li}, \bibinfo{person}{Qiang Liu}, {and} \bibinfo{person}{Jun Zhu}.} \bibinfo{year}{2024}\natexlab{}.
\newblock \showarticletitle{Fairness without Demographics on Electronic Health Records}. In \bibinfo{booktitle}{\emph{AAAI 2024 Spring Symposium on Clinical Foundation Models}}.
\newblock


\bibitem[\protect\citeauthoryear{Luo, Liu, and Liu}{Luo et~al\mbox{.}}{2022}]%
        {luo2022deep}
\bibfield{author}{\bibinfo{person}{Yingtao Luo}, \bibinfo{person}{Zhaocheng Liu}, {and} \bibinfo{person}{Qiang Liu}.} \bibinfo{year}{2022}\natexlab{}.
\newblock \showarticletitle{Deep Stable Representation Learning on Electronic Health Records}. In \bibinfo{booktitle}{\emph{2022 IEEE International Conference on Data Mining (ICDM)}}. IEEE Computer Society, \bibinfo{pages}{1077--1082}.
\newblock


\bibitem[\protect\citeauthoryear{Madras, Creager, Pitassi, and Zemel}{Madras et~al\mbox{.}}{2018}]%
        {madras2018learning}
\bibfield{author}{\bibinfo{person}{David Madras}, \bibinfo{person}{Elliot Creager}, \bibinfo{person}{Toniann Pitassi}, {and} \bibinfo{person}{Richard Zemel}.} \bibinfo{year}{2018}\natexlab{}.
\newblock \showarticletitle{Learning adversarially fair and transferable representations}. In \bibinfo{booktitle}{\emph{International Conference on Machine Learning}}. PMLR, \bibinfo{pages}{3384--3393}.
\newblock


\bibitem[\protect\citeauthoryear{Pleiss, Raghavan, Wu, Kleinberg, and Weinberger}{Pleiss et~al\mbox{.}}{2017}]%
        {pleiss2017fairness}
\bibfield{author}{\bibinfo{person}{Geoff Pleiss}, \bibinfo{person}{Manish Raghavan}, \bibinfo{person}{Felix Wu}, \bibinfo{person}{Jon Kleinberg}, {and} \bibinfo{person}{Kilian~Q Weinberger}.} \bibinfo{year}{2017}\natexlab{}.
\newblock \showarticletitle{On fairness and calibration}.
\newblock \bibinfo{journal}{\emph{Advances in neural information processing systems}}  \bibinfo{volume}{30} (\bibinfo{year}{2017}).
\newblock


\bibitem[\protect\citeauthoryear{Rahman and Purushotham}{Rahman and Purushotham}{2022}]%
        {rahman2022fair}
\bibfield{author}{\bibinfo{person}{Md~Mahmudur Rahman} {and} \bibinfo{person}{Sanjay Purushotham}.} \bibinfo{year}{2022}\natexlab{}.
\newblock \showarticletitle{Fair and interpretable models for survival analysis}. In \bibinfo{booktitle}{\emph{Proceedings of the 28th ACM SIGKDD Conference on Knowledge Discovery and Data Mining}}. \bibinfo{pages}{1452--1462}.
\newblock


\bibitem[\protect\citeauthoryear{Rawls}{Rawls}{2004}]%
        {rawls2004theory}
\bibfield{author}{\bibinfo{person}{John Rawls}.} \bibinfo{year}{2004}\natexlab{}.
\newblock \showarticletitle{A theory of justice}.
\newblock In \bibinfo{booktitle}{\emph{Ethics}}. \bibinfo{publisher}{Routledge}, \bibinfo{pages}{229--234}.
\newblock


\bibitem[\protect\citeauthoryear{Roh, Lee, Whang, and Suh}{Roh et~al\mbox{.}}{2021}]%
        {roh2021fairbatch}
\bibfield{author}{\bibinfo{person}{Yuji Roh}, \bibinfo{person}{Kangwook Lee}, \bibinfo{person}{Steven~Euijong Whang}, {and} \bibinfo{person}{Changho Suh}.} \bibinfo{year}{2021}\natexlab{}.
\newblock \showarticletitle{FairBatch: Batch Selection for Model Fairness}. In \bibinfo{booktitle}{\emph{9th International Conference on Learning Representations}}. The International Conference on Learning Representations.
\newblock


\bibitem[\protect\citeauthoryear{Sagawa, Koh, Hashimoto, and Liang}{Sagawa et~al\mbox{.}}{2019}]%
        {sagawa2019distributionally}
\bibfield{author}{\bibinfo{person}{Shiori Sagawa}, \bibinfo{person}{Pang~Wei Koh}, \bibinfo{person}{Tatsunori~B Hashimoto}, {and} \bibinfo{person}{Percy Liang}.} \bibinfo{year}{2019}\natexlab{}.
\newblock \showarticletitle{Distributionally Robust Neural Networks}. In \bibinfo{booktitle}{\emph{International Conference on Learning Representations}}.
\newblock


\bibitem[\protect\citeauthoryear{Shui, Xu, Chen, Li, Ling, Arbel, Wang, and Gagn{\'e}}{Shui et~al\mbox{.}}{2022}]%
        {shui2022learning}
\bibfield{author}{\bibinfo{person}{Changjian Shui}, \bibinfo{person}{Gezheng Xu}, \bibinfo{person}{Qi Chen}, \bibinfo{person}{Jiaqi Li}, \bibinfo{person}{Charles~X Ling}, \bibinfo{person}{Tal Arbel}, \bibinfo{person}{Boyu Wang}, {and} \bibinfo{person}{Christian Gagn{\'e}}.} \bibinfo{year}{2022}\natexlab{}.
\newblock \showarticletitle{On learning fairness and accuracy on multiple subgroups}.
\newblock \bibinfo{journal}{\emph{Advances in Neural Information Processing Systems}}  \bibinfo{volume}{35} (\bibinfo{year}{2022}), \bibinfo{pages}{34121--34135}.
\newblock


\bibitem[\protect\citeauthoryear{Wang, Zhao, Dong, Chen, Li, and Derr}{Wang et~al\mbox{.}}{2022}]%
        {wang2022improving}
\bibfield{author}{\bibinfo{person}{Yu Wang}, \bibinfo{person}{Yuying Zhao}, \bibinfo{person}{Yushun Dong}, \bibinfo{person}{Huiyuan Chen}, \bibinfo{person}{Jundong Li}, {and} \bibinfo{person}{Tyler Derr}.} \bibinfo{year}{2022}\natexlab{}.
\newblock \showarticletitle{Improving fairness in graph neural networks via mitigating sensitive attribute leakage}. In \bibinfo{booktitle}{\emph{Proceedings of the 28th ACM SIGKDD conference on knowledge discovery and data mining}}. \bibinfo{pages}{1938--1948}.
\newblock


\bibitem[\protect\citeauthoryear{Yan, Kao, and Ferrara}{Yan et~al\mbox{.}}{2020}]%
        {yan2020fair}
\bibfield{author}{\bibinfo{person}{Shen Yan}, \bibinfo{person}{Hsien-te Kao}, {and} \bibinfo{person}{Emilio Ferrara}.} \bibinfo{year}{2020}\natexlab{}.
\newblock \showarticletitle{Fair class balancing: Enhancing model fairness without observing sensitive attributes}. In \bibinfo{booktitle}{\emph{Proceedings of the 29th ACM International Conference on Information \& Knowledge Management}}. \bibinfo{pages}{1715--1724}.
\newblock


\bibitem[\protect\citeauthoryear{Zhao, Dai, Shu, and Wang}{Zhao et~al\mbox{.}}{2022}]%
        {zhao2022towards}
\bibfield{author}{\bibinfo{person}{Tianxiang Zhao}, \bibinfo{person}{Enyan Dai}, \bibinfo{person}{Kai Shu}, {and} \bibinfo{person}{Suhang Wang}.} \bibinfo{year}{2022}\natexlab{}.
\newblock \showarticletitle{Towards Fair Classifiers Without Sensitive Attributes: Exploring Biases in Related Features}. In \bibinfo{booktitle}{\emph{Proceedings of the Fifteenth ACM International Conference on Web Search and Data Mining}}. \bibinfo{pages}{1433--1442}.
\newblock


\end{thebibliography}

\end{document}